\documentclass{article} 
\usepackage[margin=3cm]{geometry}
\usepackage[utf8]{inputenc}
\usepackage{amsmath}
\usepackage{color}
\usepackage{xcolor}
\usepackage[hidelinks,colorlinks]{hyperref}
\definecolor{mydarkblue}{rgb}{0,0.08,0.45}
\hypersetup{ %
colorlinks=true,
linkcolor=blue,
citecolor=blue,
filecolor=blue,
}

\usepackage{amsthm}
\usepackage{amsfonts}
\usepackage{float}
\usepackage{graphicx}
\usepackage{bm}
\usepackage{amssymb}
\usepackage{bbm}
\usepackage{color}
\usepackage{xcolor}
\usepackage{natbib}
\usepackage{algorithm}
\usepackage{algorithmic}
\usepackage{mathtools}
\usepackage{dsfont}
\usepackage{xspace}
\usepackage{caption}
\usepackage{subcaption}

\newcommand{\eat}[1]{}
\newtheorem{theorem}{Theorem}[section]
\newtheorem{lemma}[theorem]{Lemma}
\newtheorem{definition}[theorem]{Definition}

\newtheorem{corollary}[theorem]{Corollary}

\newcommand{\R}{\mathbb{R}}

\newcommand{\X}{\mathcal{X}}

\newcommand{\mC}{\mathcal{C}}

\newcommand{\mS}{\mathcal{S}}

\newcommand{\mG}{\mathcal{G}}

\newcommand{\prog}{improvement \xspace}
\newcommand{\Prog}{Improvement \xspace}

\renewcommand{\d}{\mathrm{d}}

\newcommand{\x}{\mathbf{x}}

\newcommand{\s}{\mathbf{S}}

\newcommand{\logf}[2]{\ensuremath{\log\left(\frac{#1}{#2}\right)}}
\newcommand{\rgta}{\rightarrow}

\newcommand{\zo}{\ensuremath{\{0,1\}}}

\newcommand{\KL}[2]{D\left( #1 \| #2 \right)}

\newcommand{\infnorm}[1]{\left\lVert#1\right\rVert_\infty}

\newcommand{\lt}{\left}
\newcommand{\rt}{\right}

\newcommand{\macc}{multi-accuracy\xspace}
\newcommand{\mact}{multi-accurate\xspace}

\newcommand{\mcbn}{approximate multi-calibration\xspace}
\newcommand{\mcbd}{approximately multi-calibrated\xspace}

\newcommand{\smcbn}{multi-calibration\xspace}

\newcommand{\maxent}{MaxEnt\xspace}

\renewcommand{\varepsilon}{\epsilon}

\renewcommand{\d}{\mathrm{d}}

\newcommand{\abs}[1]{\ensuremath \Bigl\lvert #1 \Bigr\rvert}

\DeclareMathOperator{\poly}{poly}

\DeclareMathOperator*{\E}{\mathbb{E}}

\newcommand{\MC}{MC\xspace}
\newcommand{\ME}{LL-KLIEP\xspace}
\newcommand{\KLIEP}{KLIEP\xspace}
\newcommand{\ulsif}{uLSIF\xspace}

\newcounter{this-list}

\usepackage[utf8]{inputenc}
\usepackage[english]{babel}
\usepackage{natbib}
\usepackage{sidecap}

\begin{document}
\newcommand{\theTitle}{KL Divergence Estimation with Multi-group Attribution}

\author{Parikshit Gopalan\footnote{Email: \texttt{pgopalan@vmware.com}}\\
VMware Research
\and Nina Narodytska\footnote{Email: \texttt{nnarodytska@vmware.com }}\\
VMware Research
\and Omer Reingold\footnote{Most of the work performed while visiting VMware Research. Research supported in part by
NSF Award IIS-1908774. Email: \texttt{reingold@stanford.edu}}\\
Stanford University
\and Vatsal Sharan\footnote{Part of the work performed while at MIT. Email: \texttt{vsharan@usc.edu}}\\
USC
\and Udi Wieder\footnote{Email: \texttt{uwieder@vmware.com}}\\
VMware Research
}

\title{\theTitle}
\date{}

\clearpage
\maketitle
\begin{abstract}

Estimating the Kullback-Leibler (KL) divergence between two distributions given samples from them is well-studied in machine learning and information theory. Motivated by considerations of multi-group fairness, we seek KL divergence estimates that accurately reflect the contributions of sub-populations to the overall divergence. We model the sub-populations coming from a rich (possibly infinite) family $\mathcal{C}$ of overlapping subsets of the domain. We propose the notion of multi-group attribution for $\mathcal{C}$, which requires that the estimated divergence conditioned on every sub-population in $\mathcal{C}$ satisfies some natural accuracy and fairness desiderata, such as ensuring that sub-populations where the model predicts significant divergence do diverge significantly in the two distributions. Our main technical contribution is to show that multi-group attribution can be derived from the recently introduced notion of multi-calibration for importance weights \cite{hkrr2018,gopalan2021multicalibrated}. We provide experimental evidence to support our theoretical results, and show that multi-group attribution provides better KL divergence estimates when conditioned on sub-populations than other popular algorithms.

\end{abstract}

\thispagestyle{empty}
\newpage
\tableofcontents
\thispagestyle{empty}
\newpage
\setcounter{page}{1}

\section{Introduction}
\label{sec:intro}

The problem of comparing and contrasting distributions is central to machine learning.
\eat{As a demonstrative example consider a researcher that wishes to advocate for a social aid program and needs to identify the set of recipients in the population who would benefit the most from the aid. The data contains rich demographic and economic information as features that could be used for this task. The researcher considers two different ways to measure poverty\footnote{How to measure poverty is a lively debate in economic policy}. The researcher models the two measures as distributions $P$ and $R$ on the domain $\X$ of features and then measures the divergence between them.}
As a illustrative example, consider a medical researcher who has patient data from a prior outbreak ($P$) of a disease and a more recent one ($R$). The data contains medical and demographic information as features. The researcher wishes to identify patterns of shifts between the epidemiological behavior of the two outbreaks. The researcher models the two outbreaks as distributions $P$ and $R$ on a domain $\X$ of features and then measures the divergence between them.
There are several divergences studied in the literature, of which the Kullback-Leibler (KL) divergence is arguably the most important \cite{CTbook}. It is defined as
\begin{align}
    \label{eq:def-kl}
    \KL{R}{P} = \E_{\x \in R}\lt[ \logf{R(\x)}{P(\x)} \rt].
\end{align} 
The definition shows that measuring the divergence does not require  models for both $P$ and $R$, rather it suffices to have a model $w$ for the ratio $w^*(x) = R(x)/P(x)$, as the divergence is just the expectation of $\log w^*$ under $R$. The function $w$ is referred to as the importance weights of our model or the density ratio \cite{sugiyamaBook}. Indeed, while in general $w^*$ cannot be learned exactly, there are numerous algorithms in the literature that learn models $w$ for $w^*$ satsfying certain desiderata \cite{KLIEP, nguyen2010estimating, dudik2007maximum}.

The importance weights $w$ define a distribution $Q$ where $Q(x) = w(x)P(x)$ which is our model for $R$ and can be used to estimate the divergence.  Say the researcher finds that the divergence is large, and concludes that the two outbreaks differ significantly. Going further, the researcher would want to use the model $Q$ to determine how various \emph{sub-groups} of the population contribute to the divergence. 
Concretely, let $C$ be some sub-population of interest, say, people of a certain age bracket. Let $Q|_C$ (respectively $P|_C,R|_C$) be the distribution conditioned on $x \in C$. Does the model $Q|_C$ give insight into the \emph{conditional} divergence between $R|_C$ and $P|_C$? If $\KL{Q|_C}{P|_C}$ is found to be large, what can the researcher infer about $\KL{R|_C}{P|_C}$?

This is important not just for the utility of the researcher, it is also motivated by the desire for the set $C$ to be treated fairly by the model $Q$. Realistically, there could be many (possible even infinite) sub-groups of interest, and they might overlap. We model them as coming from a family $\mathcal C$ of subsets of the domain. The fundamental question that we ask, is given  $\mathcal C$, and two distributions $R$ and $P$, what does it mean to learn a model $Q$ so that the estimated divergence is attributed `fairly' to the sets in $\mathcal C$?

\subsection{Multi-group attribution}

Given sample access to two distributions $R$ and $P$ over a domain $\X$, our goal is to estimate $\KL{R}{P}$. Let $\mC = \{C: \X \rgta \zo\}$ denotes a family of subsets that contains the sub-populations we are interested in.\footnote{Assume for simplicity that $\mC$ is closed under complement.}  We want to be able to meaningfully attribute portions of our estimate of KL divergence to various subset $C \in \mC$. To formalize this, we introduce some notation. 

Given $w: \X \rgta \R^+$ where $\E_P[w(x)] =1$, let $Q$ be the distribution defined by $Q(x) = w(x)P(x)$, we denote this by $Q = w \cdot P$. For every $C \subset \X$, let $Q(C) = \Pr_Q[\x \in C]$. When $Q(C) > 0$, let $Q|_C = Q(x)/Q(C)$ for $x \in C$ denote the conditional distribution of $Q$ over $C$. We let $w^*(x) = R(x)/P(x)$ denote the ground truth importance weights. Denote the KL divergence between Bernoulli variables with parameters $p$ and $q$ by
\begin{align}
    \label{eq:bkl}
    d(p,q) = p\logf{p}{q} + (1- p)\logf{1-p}{1-q}.
\end{align} 

\paragraph{Defining attribution.} How much of the total KL divergence should one attribute to a sub-population $C \in \mC$, where $\mC$ is a large, overlapping collection of sub-populations? Our starting point is the following decomposition of the divergence, which follows from the chain rule \cite{CTbook}:

\begin{lemma}
\label{lem:chain-rule}
Let $C \subseteq \X$ and let $\bar{C} = \X \setminus C$. Then 
\begin{gather} 
\KL{R}{P} = d(R(C), P(C)) + R(C)\KL{R|_C}{P|_C}\notag + R(\bar{C})\KL{R|_{\bar{C}}}{P|_{\bar{C}}}.\notag
\end{gather}
\end{lemma}
\begin{enumerate}
\item The {\em marginal term} $d(R(C), P(C))$  accounts for differences in the measure of $C$ under $R$ and $P$.
\item We call $R(C)\KL{R|_C}{P|_C}$ the {\em conditional contribution from $C$} as it is attributable to differences between $R$ and $P$ conditioned on $C$. Similarly $R(\bar{C})\KL{R|_{\bar{C}}}{P|_{\bar{C}}}$ is
 the {\em conditional contribution from $C$}. 
\end{enumerate}

Clearly the first two terms in this decomposition are attributable to $C$. This motivates our definition of {\em ideal} attribution, which lets us estimate these terms.

\begin{definition}
\label{def:id-attr}
The distribution $Q$ satisfies ideal multi-group attribution for $(P, R, \mC)$ if for every $C \in \mC$, 
\begin{align}
    Q(C) &= R(C), \label{eq:att-1}\\
    \KL{Q|_C}{P|_C} &= \KL{R|_C}{P|_C}.\label{eq:att-2} 
\end{align} 
\end{definition}

The first condition is known as multiaccuracy \cite{gopalan2021multicalibrated}, and there are many known algorithms that achieve it. It suffices to correctly estimate the marginal term.
Having Equation \eqref{eq:att-2} as well would let us estimate the conditional contribution from $C$. Unfortunately finding $Q$  satisfying this equation (or even a reasonable approximation) is known to be information-theoretically prohibitive \cite{BatuFRSW13, Valiant11} (see Lemma \ref{lem:impossibility} which shows that any reasonable approximation of the KL-divergence requires a sample complexity polynomial in the size of the domain, which could be exponential in the dimensionality of the data). Instead we propose the following relaxation.

\begin{definition}
\label{def:attr}
The distribution $Q$ satisfies (exact) multi-group attribution for $(P, R, \mC)$ if for every $C \in \mC$, 
\begin{gather}
    Q(C) = R(C),\label{eq:att1}\\
    \KL{R|_C}{Q|_C}  + \KL{Q|_C}{P|_C} 
    = \KL{R|_C}{P|_C}.\label{eq:att2}
\end{gather}
\end{definition}

We call the second condition the Pythagorean property, in keeping with the literature \cite{CTbook}. Comparing it to Equation \eqref{eq:att-2}, there is an additional $\KL{R|_C}{Q|_C}$ term on the left. The Pythagorean property relaxes the condition that $R|_C$ and $Q|_C$ must be close (which is impossible to achieve efficiently), to asking that $Q|_C$ lies {\em in between} $R|_C$ and $P|_C$. Why is this relaxation meaningful for attribution?

\paragraph{Soundness and Improvement.} The Pythagorean property implies the following inequalities for the model $Q$ that we will refer to as {\em soundness} and {\em \prog} respectively:
\begin{align} 
\KL{R|_C}{P|_C} \geq \KL{Q|_C}{P|_C},\label{eq:att2.1}\\
\KL{R|_C}{P|_C} \geq \KL{R|_C}{Q|_C}.\label{eq:att2.2}
\end{align}
Let us explain why these are desirable conditions. 
Soundness and multiaccuracy together imply that 
\[ R(C)\KL{R|_C}{P|_C} \geq Q(C)\KL{Q|_C}{P|_C} \]
hence our model $Q$ is conservative in estimating the conditional contribution of $C$. This endows the model $Q$ with the following {\em soundness} guarantee: if it attributes large divergence to any set $C \in \mC$, the ground truth conditional contribution of $C$ is only larger. 

To see why \prog is desirable for attribution, assume it is violated for some $C \in \mC$, where $\KL{R|_C}{P|_C} < \KL{R|_C}{Q|_C}$. Intuitively, $Q$ is meant to be a reweighting of $P$ which is closer to $R$. But conditioned on $C$, it is farther from $R$ than $P$ was. Given this, it is unclear that the divergence it attributes to $C$ is meaningful. 
\Prog guarantees that $Q$ {\em simultaneously improves} on $P$ as a model for $R$ conditioned on every $C \in \mC$; which is desirable from a fairness perspective.

\subsection{Another view of the Pythagorean property}

We show how the Pythagorean property arises naturally from requiring that two estimators for the conditional KL divergence be equal. Say we have a model $Q = w \cdot P$ for $R$ where $w: \X \rgta \R^+$ are the importance weights. We wish to use it to estimate $\KL{R|_C}{P|_C}$ for some $C \in \mC$. 

A natural estimator uses $Q|_C$ in place of $R|_C$:
\begin{align}
    \KL{Q|_C}{P|_C} &= \E_{Q|_C}\left[\logf{Q(x)P(C)}{Q(C)P(x)}\right]\notag \\
     &= \E_{Q|_C}[\log(w(x)] +\logf{P(C)}{Q(C)}. \label{eq:est1}
\end{align}
This estimator is always positive by the positivity of KL divergence, but  it may violate soundness (Equation \eqref{eq:att2.1}).

A second estimator is obtained by using $w(x)$ (suitably renormalized) in place of the true weights $w^*(x)$ and appears in \cite{sugiyamaBook, KLIEP, nguyen2010estimating}:
\begin{align}
    \E_{R|_C}\left[\logf{w(x)P(C)}{Q(C)}\right] = \E_{R|_C}[\log(w(x))] + \logf{P(C)}{Q(C)}.\label{eq:est2}
\end{align}
The only difference from the previous estimator is that we compute the expectation over $R|_C$. But this estimator has rather different guarantees: we can rewrite it as
\begin{align}
    \E_{R|_C}\left[\logf{w(x)P(C)}{Q(C)}\right] &= 
    \E_{R|_C}\left[\logf{Q|_C(x)}{P|_C(x)}\right]\notag\\
   &= \E_{R|_C}\left[\logf{R|_C(x)}{P|_C(x)} - \logf{R|_C(x)}{Q|_C(x)} \right]\notag\\
    &= \KL{R|_C}{P|_C} - \KL{R|_C}{Q|_C}. \label{eq:est2.1}
\end{align} 
Since $\KL{R|_C}{Q|_C} \geq 0$, this estimator always guarantees soundness. However, unlike the previous estimator, this estimator might be negative. \Prog (Equation \eqref{eq:att2.2} captures the condition that it is non-negative. 

If the two estimators are equal, we get an estimate which is non-negative and a lower bound, and where the model $Q$ satisfies both soundness and \prog. This equivalence is characterized by the Pythagorean property.

\begin{lemma}
\label{lem:eq-est}
The following are equivalent:
\begin{enumerate}
\item The two estimators are equal. 
\item The Pythagorean property holds conditioned on $C$: $\KL{Q|_C}{P|_C} = \KL{R|_C}{P|_C} - \KL{R|_C}{Q|_C}$. 
\item $\E_{R|_C}[\log(w(x))] = \E_{Q|_C}[\log(w(x))]$. 
\end{enumerate}
\end{lemma}
\begin{proof}[Proof of Lemma \ref{lem:eq-est}]
The first estimator equals the LHS of (2) by definition (Equation \eqref{eq:est1}),  whereas the second equals the RHS of (2) by Equation \eqref{eq:est2.1}. The equivalence of (1) and (3) follows by equating the RHS of Equations \eqref{eq:est1} and \eqref{eq:est2}.
\end{proof}

\eat{\paragraph{Paper Organization}
In Section (2) we define formally the notion of 'attributable KL estimation' and set-up the problem.  }

\eat{
Say the researcher computes a model $w$ and finds that $\E_R[\log w]$ is large, the researcher concludes that the two methods to measure poverty differ significantly. Now it is clearly important to \emph{attribute} the divergence found correctly. In other words, let $C$ be some sub-population of interest, let $R_C$ (respectably $P_C$) be the distribution conditioned on $x \in C$. Can the model $w$ say something meaningful about the \emph{conditional} divergence $\KL{R_C}{P_C}$? 
We note that the number of of such sets could very large.}
\section{Overview of our results}
\label{sec:overview}

\subsection{Efficient multi-group attribution}

We address the question of whether it is possible to achieve multi-group attribution efficiently, in terms of samples and computation. Since the current definition which requires exact equalities is not possible to achieve with finitely many samples from $R$  and $P$,\footnote{Even checking that $R(C)$ and $P(C)$ are exactly equal is hard given samples.} we define the notion of $(\alpha, \beta)$ multi-group attribution which allows for some slack in both equations. Our main technical contribution is to show that multi-group attribution can be derived from the multi-calibrated partitions introduced by  \cite{gopalan2021multicalibrated}. Informally, this is a partition of the domain $\X$ into regions $S_1, \ldots, S_m$ where the condition
$R|_{S_i}(C) = P|_{S_i}(C)$ holds for every $i \in [m]$ and $C \in \mC$. This means that conditioned on the partition, no $C \in \mC$ can distinguish between $R$ and $P$. Formally, our partitions satisfy a relaxed notion of this condition, which allows them to be computed efficiently. Given such a partition, our main result, Theorem \ref{thm:att-mcab} shows that the distribution $Q = w\cdot P$ where $w(S_i) =R(S_i)/P(S_i)$ for all $x \in S_i$ satisfies multi-group attribution for $\mC$. Combined with the algorithm from \cite{gopalan2021multicalibrated}, we derive an efficient algorithm for multi-group attribution, assuming that $\mC$ has a weak agnostic learner \cite{SBD2, hkrr2018, gopalan2021multicalibrated}. 

We complement this result with a strong negative result showing that a number of well-known algorithms in the literature viz. Log-linear KLIEP \cite{KLIEP, sugiyamaBook}, the Gibbs distribution based estimator from  \cite{nguyen2007nonparametric, nguyen2010estimating} and the MaxEnt algorithm \cite{jaynes1957information, della1997inducing, KazamaT, dudik2004performance, dudik2007maximum} produce distributions $Q$ that do not satisfy multi-group attribution. While they guarantee multiaccuracy, they do not guarantee either soundness or improvement.

\subsection{Sandwiching bounds from multi-group attribution}

The work of \cite{gopalan2021multicalibrated} introduced sandwiching bounds as a multi-group fairness/accuracy desideratum for importance weights. The setting is that the learnt importance weights $w$ are meant to approximate the true importance weights $w^*$ of the distribution $R$ relative to $P$. In analogy with completeness and soundness for proof systems, they ask that for every $C \in \mC$, the importance weights satisfy
\begin{align}
\label{eq:grsw-sandwich}
\frac{R(C)}{P(C)} \leq \E_{R|_C}[w(x)] \leq \E_{R|_C}[w^*(x)].
\end{align}

We show that multi-group attribution implies similar sandwiching bounds for $\log(w)$. Formally, if $Q = w \cdot P$ satisfies Definition \ref{def:attr}, then 
\begin{align}
\label{eq:our-sandwich}
\logf{R(C)}{P(C)} \leq \E_{R|_C}[\log(w(x))] \leq \E_{R|_C}[\log(w^*(x))].
\end{align}

Equations \eqref{eq:grsw-sandwich} and \eqref{eq:our-sandwich} are similar in form, yet neither of them implies the other, and indeed there are some subtle differences. We show that the upper bound in Equation \eqref{eq:our-sandwich} only requires multi-accuracy (as opposed to full multi-group attribution). In contrast, \cite{gopalan2021multicalibrated} showed that neither direction of Equation \eqref{eq:grsw-sandwich} is implied by multiaccuracy alone.

\subsection{Experiments}
We  test  the performance of various algorithms for estimating KL divergence between distributions and the divergence when conditioned on various sub-populations, for mixtures of Gaussians and MNIST-based data. Our results are in line with what the theory predicts, even for such simple data sets: algorithms that guarantee multi-group attribution are indeed better at estimating the divergence per sub-population.

\subsection{Summary of our contributions}
We summarize our contributions below:
\begin{enumerate}
    \item We motivate the problem of fair and accurate multi-group attribution in KL divergence estimation. We present a formal definition of multi-group attribution for a class $\mC$ of sub-populations in Section \ref{sec:defs}.
    \item We show that multi-group attribution can be achieved from the multi-calibrated partitions of \cite{gopalan2021multicalibrated} in Section \ref{sec:mcab}. This implies efficient algorithm for the task of multi-group attribution, under the assumption that the class $\mC$ is weakly agnostically learnable.
    \item In contrast, in Section \ref{sec:macc} we show that a number of popular algorithms in the literature do not satisfy multi-group attribution by constructing counterexamples.
    \item  We validate our theoretical claims with experiments showing that algorithms that guarantee multi-group attribution are indeed better at estimating the divergence per sub-population in Section \ref{sec:experiments}.
\end{enumerate}
We discuss additional related work in Section \ref{sec:related}. Some additional proofs and lemmas are in Appendix~\ref{app:mcab}. 

\section{Notation and Definitions}
\label{sec:defs}

\paragraph{Notation.} All distributions are over a domain $\X$. We work with discrete domains for simplicity, but all results can be generalized to the continuous setting. Given two distribution $P, R$ on domain $\X$, our goal is to estimate the KL divergence between them. 
We use $\mC = {\{C: \X \rgta \zo\}}$ to denote a set of sub-populations for which we want attribution guarantees. We use the notation $Q = w \cdot P$ to denote that a distribution $Q$ on $\X$ has importance weights $w(x) = Q(x)/P(x)$. We say the weights are explicit, if the function $w:\X \rgta \R$ is computable efficiently; this does not necessarily require the pdf of $Q$ to be explicit. For importance weights $w$, let 
\[ \infnorm{w} = \max_{x \in \X}(w(x), 1/w(x)).\] 
For $Q =w\cdot P$, we have $\KL{Q}{P} \in [0, \log(\infnorm{w})$. Let $Q(C) = \Pr_{Q}[\x \in C]$, and let $Q|_C$ denote $Q$ conditioned on $C$. We let  $R = w^*\cdot P$.  

\paragraph{Multi-group attribution.} We present a relaxation of Definition \ref{def:attr} that allows slack in the equalities, which is inherent as we only get sample access to the distributions.

\begin{definition}
\label{def:ab-attr}
Let $\alpha, \beta \geq 0$. The distribution $Q$ satisfies $(\alpha, \beta)$  multi-group attribution for $(P, R, \mC)$ if for all sets $C \in \mC$, 
\begin{gather}
    \abs{Q(C) - R(C)} \leq \alpha, \label{eq:ab-att1}\\
    \abs{\KL{R|_C}{Q|_C}  + \KL{Q|_C}{P|_C} - \KL{R|_C}{P|_C}} \leq \frac{\beta}{R(C)}. \label{eq:ab-att2}
\end{gather}
\end{definition}
We refer to these conditions as approximate multiaccuracy and the approximate Pythagorean property respectively. 
In the RHS of Equation \eqref{eq:ab-att2}, we normalize by $R(C)$. This means that Pythagorean property is meaningful only when $R(C)$ is not too small. This is inevitable in the sample access setting, where we cannot get meaningful bounds for very small sets in the Pythagorean property (in contrast to multiaccuracy). To see why, suppose we take $O(1/\alpha^2)$ samples from $R$ and $P$, and don't see any samples lying in $C$. We can be confident that $R(C), P(C) \leq \alpha$, hence approximate mutliaccuracy holds. However, we cannot say anything about  $\KL{R|_C}{P|_C}$, which could be anywhere in $[0, \log(\infnorm{w^*}))$.   

The reason we use two approximation parameters is that they are of different scale. Being the difference of probabilities, $\alpha \in [0,1]$.
It can be shown  $\beta \in [0, 2\log(\infnorm{w^*})$ for models $Q$ where $\infnorm{w} \leq \infnorm{w^*}$. One can think of $\alpha, \beta$ as parameters that control the sample complexity. To achieve smaller values of $\alpha, \beta$ we need more samples, but this lets us get stronger guarantees, and reason about smaller sets $C$. Our bounds are meaningful for $C$ when $R(C) \geq \alpha$.

The approximate Pythagorean property \eqref{eq:ab-att2} implies approximate versions of soundness and improvement:
\begin{align} 
\KL{R|_C}{P|_C} \geq \KL{Q|_C}{P|_C} - \frac{\beta}{R(C)},\label{eq:ab-att2.1}\\
\KL{R|_C}{P|_C} \geq \KL{R|_C}{Q|_C} - \frac{\beta}{R(C)}.\label{eq:ab-att2.2}
\end{align}

\paragraph{Multiaccuracy.} If the distribution $Q$ satisfies Equation \eqref{eq:ab-att1}, we say it is $\alpha$-multiaccurate for $(R, \mC)$. The set of all $\alpha$-multiaccurate distributions forms a convex polytope that we denote by $K^\alpha(R, \mC)$.

\paragraph{Partitions and multicalibration.}
The following use of partitions to define importance weights is from \cite{gopalan2021multicalibrated}. 

\begin{definition}
A partition $\mS = \{S_1, \ldots, S_m\}$ of $\X$ is a collection of disjoint sets whose union is $\X$. Given  distributions $R, P$, the $(R, \mS)$-reweighting of $P$ is the distribution $Q = w\cdot P$ whose importance weights are $w(x) = w(S_i) = R(S_i)/P(S_i)$ for $ x\in S_i$. 
\end{definition}
The above importance weights satisfy $\infnorm{w} \leq \infnorm{w^*}$, indeed $\infnorm{w}$ might be bounded even if $\infnorm{w^*}$ is not.

Every distribution $Q$ on $\X$ induces a distribution on $\mS$, let $\s \sim Q$ denote sampling from $\mS$ according to it.

\begin{lemma}
\label{lem:partition}
Let $Q$ be the $(R, \mS)$ reweighting of $P$. Then $Q$ and $R$ induce identical distributions on $\mS$. For every $i \in [m]$, $Q|_{S_i} = P|_{S_i}$.
\end{lemma}
\begin{proof}[Proof of Lemma \ref{lem:partition}]
    Note that for every $i \in [m]$, 
    \[ Q(S_i) = \sum_{x \in S_i} P(x)w(S_i) = P(S_i) \frac{R(S_i)}{P(S_i)}  = R(S_i).\]
    Under $Q$ every $x \in S_i$, has the same importance weight $w(S_i)$ relative to $P$, hence the conditional distributions $Q|_{S_i}$ and $P|_{S_i}$ are identical. 
\end{proof}

This lemma gives a natural coupling of $Q$ and $R$: sample $\s \sim R$, and then sample $\x \sim P|_{\s}$ to generate a sample from $Q$ and $\x' \sim R|_{\s}$ to generate a sample from $R$. 

We next define the notion of multi-calibrated partitions.

\begin{definition}
\label{def:mcab}
    Let $\mC$ be a collection of subsets of $\X$.
    We say that the partition $\mS$ is $\alpha$-\mcbd for $(P, R, \mC)$  if for all $C \in \mC$,
    \begin{align}
    \label{eq:def-mcab}
        \E_{\s \sim R}\big[\big|{R|_{\s}(C) - P|_\s(C)\big|}\big] \leq \alpha. 
    \end{align}
\end{definition}

The original notion of $\alpha$-\smcbn from \cite{gopalan2021multicalibrated}.
requires that for every $i \in [m]$ and $C \in \mC$
    \begin{align*}
        \big|R|_{S_i}(C) - P|_{S_i}(C)\big| \leq \alpha. 
    \end{align*}
Our notion of $\alpha$-\mcbn is weaker, since it only requires closeness of the conditional distributions on average, hence it is implied by $\alpha$-\smcbn. For $\mC$ which is weakly agnostically learnable, the algorithm for $(\alpha, \beta)$-multi-calibration in \cite{gopalan2021multicalibrated} can be used to compute an $\alpha$-\mcbd partition by setting $\beta = 1/\infnorm{w^*}$.  The number of states is $\poly(\log(\infnorm{w^*}), 1/\alpha)$, and the running time is in time $\poly(\infnorm{w^*}, 1/\alpha)$. We refer the reader to \cite{SBD2, hkrr2018, gopalan2021multicalibrated} for the definition of weak agnostic learning and a discussion of when it is reasonable.  

\section{Attribution from multicalibration}
\label{sec:mcab}

The main theorem in this section is the following:

\begin{theorem}
\label{thm:att-mcab}
If $\mS$ is $\alpha$-\mcbd\ for $(P, R, \mC)$, then the $(R, \mS)$ reweighting of $P$, $Q = w \cdot P$  satisfies $(\alpha, \beta)$ multi-group attribution for $\mC$ where $\beta = 2\alpha \log(\infnorm{w})$.
\end{theorem}

Let $Q$ be the $(R, \mS)$ reweighting of $P$. It is easy to show that \mcbn implies multi-accuracy (Lemma \ref{lem:exp-close} in Appendix \ref{app:mcab}).
The crux is to show the Pythagorean property. 
We describe the main technical ideas used in the proof. Multicalibration guarantees the closeness of the probability of belonging to $C$ under $Q$ and $R$ conditioned on a (random) set $S_i$. The following lemma instead conditions on $C \in \mC$ and considers the distributions induced by $R|_C$ and $Q|_C$ on the sets $S_i$ in the partition, and  shows that  they are close assuming multicalibration. 

\begin{lemma}
\label{lem:stat}
     For every  $C \in \mC$, we have
     \begin{align}
         \label{eq:sd}
         \lt|\sum_{i \in [m]}R|_C(S_i) - Q|_C(S_i) \rt| \leq \frac{2\alpha}{R(C)}.
     \end{align}
\end{lemma}
\begin{proof}
    By Lemma \ref{lem:partition} we have 
    \[ Q(S_i \cap C) = R(S_i) P|_{S_i}(C), \ \ R(S_i \cap C) = R(S_i)R|_{S_i}(C).\] 
    Hence
    \begin{align*}
     \lt|\frac{Q(S_i \cap C)}{Q(C)} -  \frac{R(S_i \cap C)}{R(C)}\rt|   &= R(S_i)\left|\frac{P|_{S_i}(C)}{Q(C)} - \frac{R|_{S_i}(C)}{R(C)} \right|\\
        & \leq \frac{R(S_i)}{R(C)}\abs{P|_{S_i}(C) - R|_{S_i}(C)} + R(S_i)P|_{S_i}(C)\left|\frac{1}{Q(C)} - \frac{1}{R(C)}\right|\\
        & \leq \frac{R(S_i)}{R(C)}\abs{P|_{S_i}(C) - R|_{S_i}(C)} +
        Q(S_i \cap C)\frac{\alpha}{Q(C)R(C)}
    \end{align*}
    where we use $|Q(C) - R(C)| \leq \alpha$ by Lemma \ref{lem:exp-close}. 
    We use this to bound the LHS of Equation \eqref{eq:sd} as
    \begin{align*}
             \lt|\sum_{i \in [m]}Q|_C(S_i) - R|_C(S_i) \rt| &\leq  \sum_{i \in [m]} \lt|\frac{Q(S_i \cap C)}{Q(C)} -  \frac{R(S_i \cap C)}{R(C)}\rt|\\
            & \leq \sum_{i \in [m]}\frac{R(S_i)}{R(C)}\abs{P|_{S_i}(C) - R|_{S_i}(C)} + \sum_{i \in [m]} Q(S_i \cap C)\frac{\alpha}{Q(C)R(C)}\\
            & \leq \frac{\alpha}{R(C)} + Q(C)\frac{\alpha}{Q(C)R(C)} = \frac{2\alpha}{R(C)}.
    \end{align*}
where the last line uses the definition of \mcbn.
\end{proof}

Let us sketch how this helps prove the Pythagorean property. By Lemma \ref{lem:eq-est}, It suffices to show that the random variable $\log(w)$ has similar expectations under $R|_C$ and $Q|_C$. By the definition of $w$, $\log(w)$ is constant on each $S_i \in \mS$. Since $Q$ and $R$ induce statistically close distributions on $\mS$ and $\log(w)$ is a bounded by $\log(\infnorm{w})$ we can bound the difference in expectation. Formally, we prove the following bound:

\begin{lemma}
\label{lem:pyth}
For every  $C \in \mC$,
\begin{gather*}
    \abs{\KL{R|_C}{P|_C} - \KL{R|_C}{Q|_C} - \KL{Q|_C}{P|_C}} \leq \frac{2\alpha}{R(C)}\log(\infnorm{w}).
\end{gather*}
\end{lemma}
\begin{proof}
By Equations \eqref{eq:est2} and \eqref{eq:est2.1} 
\begin{align*}
    \KL{R|_C}{P|_C} - \KL{R|_C}{Q|_C} &=  
    \E_{\s \sim R|_C}[\log(w(\s))] + \logf{P(C)}{Q(C)}  \\
\end{align*}
By Equation \eqref{eq:est1}
\begin{align*}
    \KL{Q|_C}{P|_C} &= \E_{\x \sim Q|_C} \left[ \logf{Q|_C(\x)}{P|_C(\x)}\right] = \E_{\s \sim Q|_C}[\log(w(\s))] + \logf{P(C)}{Q(C)}
\end{align*}
    Subtracting we get 
    \begin{align*}
    \abs{\KL{R|_C}{P|_C} - \KL{R|_C}{Q|_C} - \KL{Q|_C}{P|_C}} &=\abs{\E_{\s \sim R|_C}[\log(w(\s)] - \E_{\s \sim Q|_C}[\log(w(\s)]} \\
    & \leq \lt|\sum_{i \in [m]}R|_C(S_i) - Q|_C(S_i) \rt|\max_{i \in [m]}|\log(w(S_i))|\\
    &\leq 2\alpha\frac{\log(\infnorm{w})}{R(C)}.
    \end{align*}
    where we use the $(1, \infty)$ version of Holder's inequality, and then Lemma \ref{lem:stat}.
\end{proof}

As discussed before, the degradation for small $R(C)$ is expected. Similarly, some dependence on $\infnorm{w}$ is to be expected, since if $\infnorm{w}$ is unbounded, then so are $\KL{R}{P}$ and $\KL{Q}{P}$. The proof of Theorem \ref{thm:att-mcab} is immediate from Lemmas \ref{lem:exp-close} and \ref{lem:pyth}.

\subsection{Sandwiching from attribution}

We show that sandwiching bounds for $\log(w)$ are implied by multi-group attribution, in analogy with the bounds for $w$ proved in \cite{gopalan2021multicalibrated} (see Equation \eqref{eq:grsw-sandwich}). Neither bound implies the other, indeed our upper bound only uses multiaccuracy, whereas it was shown that multiaccuracy does not imply either direction of \cite{gopalan2021multicalibrated}. The connection to the Pythagorean property makes our proof technically simpler.

\begin{corollary}
\label{cor:sandwich}
If  $Q$ satisfies $(\alpha, \beta)$ multi-group attribution for $(P, R, \mC)$, then for every  $C \in \mC$ where $R(C) > \alpha$,
\begin{align}
\logf{R(C)}{P(C)}  - \frac{\beta}{R(C)} - \frac{\alpha}{R(C) - \alpha}  &\leq \E_{R|_C} [\log(w(\x))]\leq \E_{R|_C}[\log(w^*(\x))] + \frac{\alpha}{R(C)}.
\end{align}
\end{corollary}
\begin{proof}
We start by relating the central quantity to conditional KL divergence. For $x \in C$,
\begin{align*}
    w(x) = \frac{Q(x)}{P(x)}  = \frac{Q|_C(x)}{P|_C(x)} \frac{Q(C)}{P(C)}. 
\end{align*}
Hence 
\begin{align}
    \label{eq:to-kl}
    \E_{R|_C}[\log(w(\x))] &= \E_{R|_C}\logf{Q|_C(\x)}{P|_C(x)} + \logf{Q(C)}{P(C)}.
\end{align}
We can upper bound this as
\begin{align*}
    \E_{R|_C}\lt[\logf{Q|_C(\x)}{P|_C(x)}\rt] &= \KL{R|_C}{P|_C} - \KL{R|_C}{Q|_C} \\
    & \leq \KL{R|_C}{P|_C}\\
    & = \E_{R|_C}[\log(w^*(x))] + \logf{P(C)}{R(C)}.
\end{align*}
Hence using this in Equation \eqref{eq:to-kl},
\begin{align}
\E_{R|_C}[\log(w(\x))] & \leq  \E_{R|_C}[\log(w^*(x))] +\logf{P(C)}{R(C)} + \logf{Q(C)}{P(C)} \notag \\
&= \E_{R|_C}[\log(w^*(x))] +\logf{Q(C)}{R(C)}. \label{eq:ub-p1}
\end{align}
By multiaccuracy, $Q(C) \leq R(C) + \alpha$. Hence 
\begin{align}
    \logf{Q(C)}{R(C)} \leq \log\lt(1 + \frac{\alpha}{R(C)}\rt) \leq  \frac{\alpha}{R(C)}. \label{eq:ub-p2}
\end{align}
Plugging Equation \eqref{eq:ub-p2} into \eqref{eq:ub-p1} gives the upper bound.

To show the lower bound, we start from Equation \eqref{eq:to-kl}. We lower bound the first term using the Pythagorean property as 
\begin{align}
\E_{R|_C}\logf{Q|_C(\x)}{P|_C(x)} & = \KL{R|_C}{P|_C} - \KL{R|_C}{Q|_C}\notag\\
& \geq \KL{Q|_C}{P|_C} - \frac{\beta}{R(C)} \geq - \frac{\beta}{R(C)}\label{eq:lb-p1}
\end{align}
where the last inequality uses the non-negativity of KL divergence.
For the last  term, we use $Q(C) \geq R(C) - \alpha$ and the inequality $\log(1- x) \geq -x/(1-x)$ to get
\begin{align*}
 \logf{Q(C)}{R(C)} \geq \log\lt(1 - \frac{\alpha}{R(C)}\rt) \geq -\frac{\alpha}{R(C) - \alpha} 
\end{align*}
hence
\begin{align}
 \logf{Q(C)}{P(C)} = \logf{R(C)}{P(C)} + \logf{Q(C)}{R(C)}  \geq \logf{R(C)}{P(C)} -\frac{\alpha}{R(C) - \alpha}. \label{eq:lb-p2}
\end{align}

Plugging Equations \eqref{eq:lb-p1} and \eqref{eq:lb-p2}  into Equation \eqref{eq:to-kl} gives the lower bound 
\begin{align*}
    \E_{R|_C}[\log(w(\x))] \geq \logf{R(C)}{P(C)} -\frac{\beta}{R(C)} -\frac{\alpha}{R(C) - \alpha}.
\end{align*}
\end{proof}

\section{(No) Attribution from multi-accuracy}
\label{sec:macc}

We give an explicit example that shows that a number of popular algorithms in the literature do not satisfy multi-group attribution. First we present these algorithms and show they are essentially equivalent (an observation that we treat as folklore though we have not seen it stated explicitly). Since these algorithm guarantee multi-accuracy, this implies that multi-accuracy by itself not imply multi-group attribution.

 Recall  $K^\alpha = K^\alpha(R, \mC)$ is the set of all $\alpha$-\mact distributions for $R, \mC$. For every $\mC$ and $\alpha \geq 0$, $K^\alpha$ is a convex set, since it is given by linear constraints, and it is non-empty since $R \in K^\alpha$. Another important class of distributions is the set of Gibbs distributions. 
\begin{definition}
The set of all {\em Gibbs distributions} $\mG = \mG(P, \mC)$ is all distributions of the form
\begin{align}
    Q(x) &= P(x) \exp\Big(\sum_{c \in \mC} \lambda_c c(x) -\lambda_0\Big)  \label{eq:gibbs}
\end{align}
\end{definition}
where we use $c(x)$ to denote the indicator function for the set $c$. Writing $Q =w \cdot P$, we have $\log(w(x))= \sum_{c \in \mC} \lambda_c c(x) -\lambda_0$. 
The free parameters are $\lambda_\mC =\{\lambda_c\}_{c \in \mC}$, from these, we set the parameter $\lambda_0$ so that $\E_{P}[w(\x)] =1$. We define $\ell_1(Q) = \sum_{c \in \mC}|\lambda_c|$ to be the $\ell_1$ norm of the free parameters.
\eat{
\begin{align}
\lambda_0 &= \log\lt(\E_{P}\lt[\exp\lt(\sum_{c \in \mC} \lambda_c c(x)\rt)\rt]\rt)\label{eq:gibbs2}
\end{align}
}
We now describe three approaches in the literature that lead to essentially the same algorithm, which finds a \mact Gibbs distribution. This algorithm is known to out-perform other density-ratio estimation algorithms in the non-realizable setting \cite{kanamori2010theoretical}.
\begin{enumerate}
    \item \cite{KLIEP, sugiyamaBook} {\bf Log-linear KLIEP}: Find the Gibbs distribution $Q \in \mG$ that minimizes $\KL{R}{Q}$. This goal is find a good density-ratio estimate.
    \item \cite{nguyen2007nonparametric, nguyen2010estimating} {\bf Divergence estimation using Gibbs distributions:}  Find the Gibbs distribution $Q = w\cdot P$ that maximizes the lower bound $\E_R[\log(w(\x))]$ on $\KL{R}{P}$. This algorithm is proposed in Section 4A of \cite{nguyen2010estimating} for the goal of divergence estimation.
    \item \cite{jaynes1957information, della1997inducing, KazamaT, dudik2004performance, dudik2007maximum} {\bf MaxEnt:} Learn a model $Q^\alpha \in K^\alpha(R, \mC)$ for $R$ by finding  the distribution $Q^\alpha $ that minimizes $\KL{Q}{P}$. 
\end{enumerate}

The equivalence of (1) and (2) is well-known (it follows from Equation \eqref{eq:est2.1}). A generalization to $f$-divergences may be found in \cite{sugiyama2012adensity}, or Section 7.3.1 of \cite{sugiyamaBook}. We have not found the equivalence to (3) explicitly in the literature, but  it follows from known convex duality results. 

\begin{lemma}
    \label{lem:equiv}
    \cite{dudik2007maximum}
    $Q^\alpha \in K^\alpha \cap \mG$ is the optimal solution to the following programs
    \begin{gather} 
        \min_{Q \in K^\alpha} \KL{Q}{P},\label{eq:primal}\\
        \min_{Q \in \mG} \KL{R}{Q} + \alpha \ell_1(Q).\label{eq:dual}
    \end{gather}
\end{lemma}
The first program is the one solved by \maxent. The second is an $\ell_1$-regularized version of the program considered by Log-linear KLIEP and \cite{nguyen2010estimating}. We derive their exact program by setting $\alpha =0$. \eat{Without regularization, the optimal might lie in the closure of $\mG$, so $\ell_1(Q)$ could be unbounded.} The Pythagorean property is known for the distributions $Q^\alpha$ when $\alpha =0$ without conditioning on $\mC$ (see \cite{della1997inducing}). 
\eat{
In the case of $\alpha =0$, the classic work of \cite[Proposition  4]{della1997inducing} implies an exact Pythagorean theorem. For $\alpha > 0$, the Pythagorean inequality for projections onto convex sets \cite[Theorem 11.6.1]{CTbook} implies $\KL{R}{P} \geq \KL{R}{Q^\alpha} + \KL{Q^\alpha}{P}$.
Theorem \ref{thm:pyth} in Appendix \ref{app:macc} proves the following exact bound for all $\alpha$ which subsumes both these results, and implies an approximate Pythagorean theorem for $(P, Q^\alpha, R)$:
    \begin{align}
    \label{eq:pyth-exact}
        \KL{R}{P} = \KL{R}{Q^\alpha} + \KL{Q^\alpha}{P} + \alpha \ell_1(Q^\alpha).
    \end{align}}
    
In contrast, we show that no such bound holds when conditioned on $\mC$; in fact even the soundness and \prog conditions implied by the Pythagorean property do not hold for $Q^\alpha$. We state our result for the case $\alpha =0$, it can be extended for any $\alpha \in [0, 1/4)$. We have not attempted to optimize the constant in the lower bound.  Recall the definition of $d(p,q)$ in Equation \eqref{eq:bkl}. 

\begin{theorem}
\label{thm:gap-1}
There exist distributions $P, R$ on $\zo^2$, a collections of sets $\mC$ and $C \in \mC$ where $R(C) = 1/2$ s.t.  
\begin{gather*}
    \label{eq:p-simple2}
    \KL{R|_C}{P|_C} =0, \\ \KL{R|_C}{Q^0|_C} = \KL{Q^0|_C}{P|_C} = d(3/4, 1/2) .
\end{gather*}
So $Q^0$ does not satisfy $(\alpha, \beta)$-multigroup attribution for $\mC$ for any $\beta < d(3/4,1/2)$.
\end{theorem}

\begin{proof}
Let $P$ be the uniform distribution on $\zo^2$. Let $R$ be the distribution where
\[ R(00) = 1/4, R(01) = 1/4, R(10) = 0, R(11) =1/2. \]
We denote the two coordinates $x_0, x_1$, and let $\mC$ consist of all subcubes of dimension $1$. Hence $\mC = \{x: x_i =a\}_{i \in \zo, a \in \zo}$. 

 The distribution $Q^0$ for $\alpha = 0$ is the product distribution which matches the marginal distributions on each coordinate: $Q^0(x_0 =1) = 1/2, Q^0(x_1 = 1) = 3/4$, and the coordinates are independent. The \macc\ constraints hold since 
 \[ Q(x_0 =1) = R(x_0 =1) =1/2\]
 \[ Q(x_1 = 1) = R(x_1 = 1) = 3/4 \] 
 and $Q^0$ is the maximum entropy distribution satisfying these constraints since the co-ordinates are independent. 
 
 Now consider the set $C =\{x_0 =0\}$. Let $B(p)$ denote the Bernoulli distribution with parameter $p$. It follows that $R|_C = P|_C = B(1/2)$, whereas $Q^0|_C = B(3/4)$. Hence 
 \[ \KL{R|_C}{P|_C} = 0, \KL{R|_C}{Q|_C} = \KL{Q|_C}{P|_C} = d(3/4, 1/2). \] 
 Assume that $Q$ satsifies $(\alpha,\beta)$ multi-group attribution, so that
 \begin{align*}
     2\cdot d(3/4, 1/2) = \abs{\KL{R|_C}{P|_C} - \KL{R|_C}{Q|_C} - \KL{Q|_C}{P|_C}} \leq \frac{\beta}{R(C)} 
 \end{align*}
 Since $R(C) =1/2$, this implies $\beta \geq d(3/4, 1/2)$ as desired.
 \end{proof}

Thus $0$-multiaccuracy does not imply $\beta$ close to $0$ for $(\alpha, \beta)$ multi-group attribution. In contrast, $0$-approximate multicalibration implies $\alpha = \beta =0$ by Theorem \ref{thm:att-mcab}.

\eat{
Hence the Pythagorean property is violated as
\begin{multline*} \ \abs{\KL{R|_C}{P|_C} - \KL{R|_C}{Q^0|_C} - \KL{Q^0|_C}{P|_C}}  \geq B. 
\end{multline*}
Note that the weaker properties as expressed in Equations \eqref{eq:att2.1},\eqref{eq:att2.2} are violated as well.

Our counterexample is similar to counterexamples in \cite{gopalan2021multicalibrated}, it constructs a small example on $\zo^2$ that shows a constant gap, and then uses tensoring to amplify the gap. 
}

\section{Experiments} \label{sec:experiments}
We  test  the performance of various algorithms for estimating KL divergence between distributions and the divergence when conditioned on various sub-populations, for mixtures of Gaussians and MNIST-based data. Our results are in line with what the theory predicts: multi-calibration provides a more accurate  estimate of the divergence per sub-population. We model the class $\mC$ as a collection of base classifiers with a learning algorithm.\footnote{The code for all our experiments can be found at \url{https://github.com/vatsalsharan/multigroup-kl}.}

\paragraph{The Algorithms.}
	We test four algorithms: (i) \KLIEP with code taken from \cite{KLIEP}. \KLIEP is a regression over the basic features of the data with a KL loss function; (ii) \ulsif from \cite{kanamori2009least}, which is a regression with least squares loss; (iii) an implementation of \ME aka log-linear \KLIEP as described in Section~\ref{sec:macc} (denoted by LLK); (iv) an implementation of  approximate multi-calibration as described in \cite{gopalan2021multicalibrated} (denoted by MC). The algorithm is an adaptation of the Boosting by branching programs algorithm \cite{mansour2002boosting, KearnsM99} to importance weight estimation. We implemented \ME and \MC in Python. While \KLIEP and \ulsif regress over the basic features of the data, both \ME and \MC work with any  family of binary base classifiers, effectively using the outputs of all classifiers in the family as the feature set. We test our algorithms using $3$ different families of base classifiers, all using the standard implementation in sklearn: (1) Random Forest with depth 5 and  $10$ estimators (RF5), (2) a threshold function over a single feature (DT1, this is a decision tree with depth $1$), (3) logistic regression (LR, we only test this for the MNIST data). We denote MC and LLK with logistic regression as MC-LR, LLK-LR and so on.

\emph{On the choice of $\mC$:}
Our goal is to test the performance of the algorithms, and not necessarily the expressiveness of the class $\mC$. Thus it makes sense to use relatively simple classifiers as our model for $\mC$ and see how the algorithms perform in attributing divergence. Simple classes like conjunctions of base features and shallow decision trees are of interest for $\mC$ form a fairness perspective.

\eat{Indeed, even a decision tree of depth $1$ (which is just a threshold function over a basic feature) produces surprisingly good results. }

\eat{Further, our synthetic data-set is simply a mixture of standard Gaussians, so one should do well with simple classifiers.  }

\paragraph{Mixtures of Gaussians.}
Our goal is to create a simple pair of distributions $P,R$, where the notion of sub-populations is natural. To this end we set $P$ to be a mixture of $k$ equally-weighted $d$-dimensional Gaussians for $k = \{5,10,15\}$ and $d=2$. Each component in the mixture has identity covariance, and their means are sampled from a Normal distribution with variance $k\cdot10000\cdot I$, therefore all the components are far apart. $R$ is similar to $P$ but with the means of all the Gaussians translated by $2.5$ along both coordinates.  See Fig.~\ref{fig:gaussian_data} for an example. We refer to each Gaussian in $P$ and its shifted counterpart in $Q$ as a pair. The divergence between every pair of Gaussians (such as those in the box in Fig.~\ref{fig:gaussian_data}) can be calculated by the closed form expression for the KL-divergence between shifted Gaussians, and is  $\approx 9$ for our parameters. Since the Gaussians are well separated (the KL divergence between any two components of $P$ is $>500$) this is also a good estimate for $\KL{R}{P}$. We generate $N=30000$ samples from $P$ and $R$ to train the algorithms. We report results for larger values of $N$ and $d$ in Appendix~\ref{sec:app-experiments} (we find them to be consistent with the results in this section).
We first test how well the algorithms estimate $\KL{R}{P}$, without conditioning on sub-populations. The results  in Table~\ref{tbl:kl-mixture-1} show that MC performs well consistently for  $k$ and with different base classifiers.

\begin{figure}
    \centering
    \begin{subfigure}{0.4\textwidth}
             \centering
    \includegraphics[width=\textwidth]{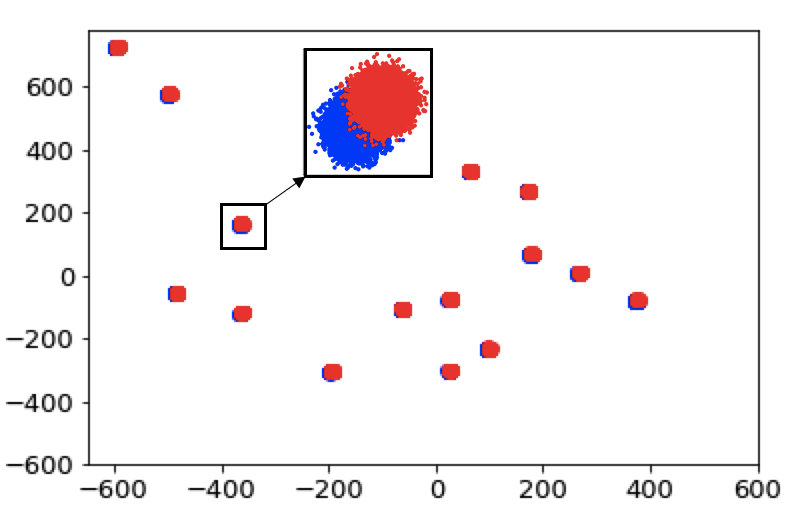}
   \caption{A mixture of $k=15$ Gaussians. $P$ is blue and $R$ is red. The zoomed-in box shows a sub-population---a pair of Gaussians from $P$ and $R$.}    \label{fig:gaussian_data}
   \end{subfigure}
   \begin{subfigure}{0.5\textwidth}
    \centering
    \includegraphics[width=\textwidth]{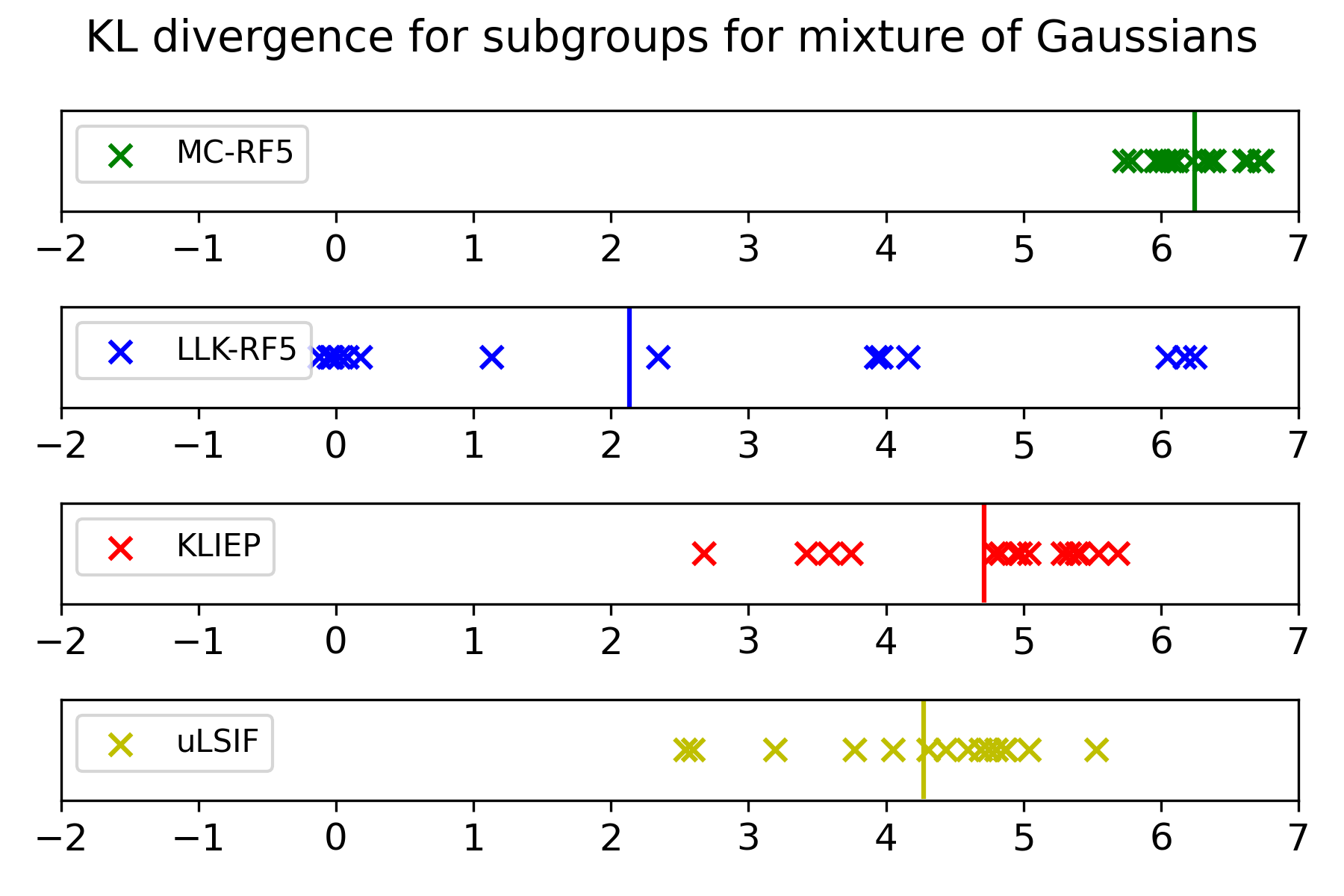}
    \caption{ KL estimates for the $15$ sub-populations.}
    \label{fig:muxture-sub-1}
\end{subfigure}
 \caption{KL estimation for mixtures of Gaussians, $k=15, d=2, N=30000$.}
\end{figure}

	\begin{table*}[h] 
		\caption{ KL estimation for mixture of $k$ Gaussians with $N=30000$, averaged over 5 trials}\label{tbl:kl-mixture-1}
		\small
		\centering
		\begin{tabular}{ c c  c c c c c c }
			\hline
			$k$ & KL & LLK-RF5 & MC-RF5 & LLK-DT1 & MC-DT1 & KLEIP & uLSIF\\
			\hline
			\hline
			5  & 9.02 & $6.43 \pm 0.26$  & $6.59 \pm 0.22$   & $1.97 \pm 0.31$   & $6.38 \pm 0.51$   & $5.41 \pm 0.21$ & $6.68 \pm 0.01$ \\
			10  & 9.02  & $3.64 \pm 1.00 $  & $6.03 \pm 0.36$    &  $0.68 \pm 0.08$ & $6.05 \pm 0.33$ & $5.02 \pm 0.21$ & $5.22 \pm 0.10$\\
			15  & 9.02    & $2.23 \pm 0.27$  & $5.92 \pm 0.04$   & $0.36 \pm 0.10$  & $5.04 \pm 0.58$ & $4.96 \pm 0.42$ & $4.39 \pm 0.14$\\
			\hline
		\end{tabular}
	\end{table*}

Once the model had been fitted, we use it \emph{without retraining} to estimate the divergence between pairs of Gaussians in the mixture. Formally this corresponds to conditioning on a set $C$ which is a bounding box around the pair of Gaussians, as demonstrated in Fig.~\ref{fig:gaussian_data}. Since the Gaussians are far apart and have negligible overlap in their densities, $\KL{R_C}{P_C}$ is very close to the closed form divergence calculated earlier between the two shifted Gaussians (${\approx9}$). Results are summarized in Fig.~\ref{fig:muxture-sub-1},  where the $x$-axis indicates the divergence, and each mark indicates a pair of sub-populations, with the horizontal line indicating the average. We see that \MC estimates the divergence across all sub-populations well, while all others have a fair bit of variance in their estimates for each pair and seem to miss a lot of the divergence from some pairs. Also, as the theory predicts, all estimates are \emph{lower bounds} on the true divergence. 

Figures \ref{fig:contour_me}, \ref{fig:contour_mc} show contours of the first $9$ pairs of Gaussians (first $9$ sub-populations) in the mixture of $k=10$ Gaussians. In these images we sampled from a mixture of $R$ and $P$ and colored the points by their importance weight. We can clearly see that \MC is  consistent in assigning the importance weights, approximately separating the cluster by a diagonal hyper-plane. On the other hand, \ME makes mistakes on a number of clusters and does not always succeed in separating the Gaussians well.
	
We report additional results in Appendix \ref{sec:app-experiments} including the standard deviation of the divergences across the sub-populations (with error bars), and experiments for some larger values of $N$ and $d$.

\begin{figure}
    \centering
    \begin{subfigure}{0.45\textwidth}
             \centering
    \includegraphics[width=\textwidth]{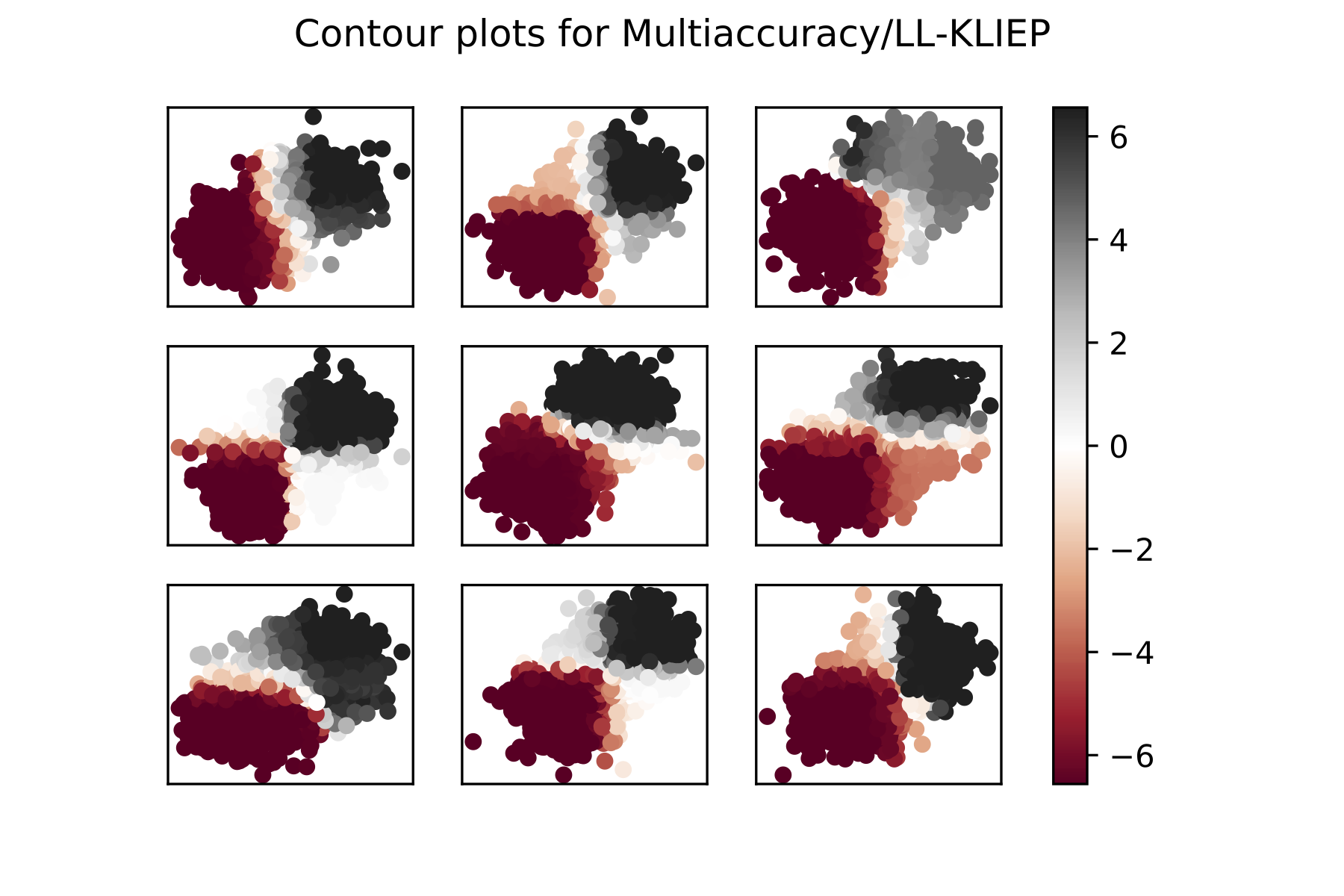}
   \caption{Color indicates the importance weight\\ assigned by LL-KLIEP.}    \label{fig:contour_me}
   \end{subfigure}
   \begin{subfigure}{0.45\textwidth}
    \centering
    \includegraphics[width=\textwidth]{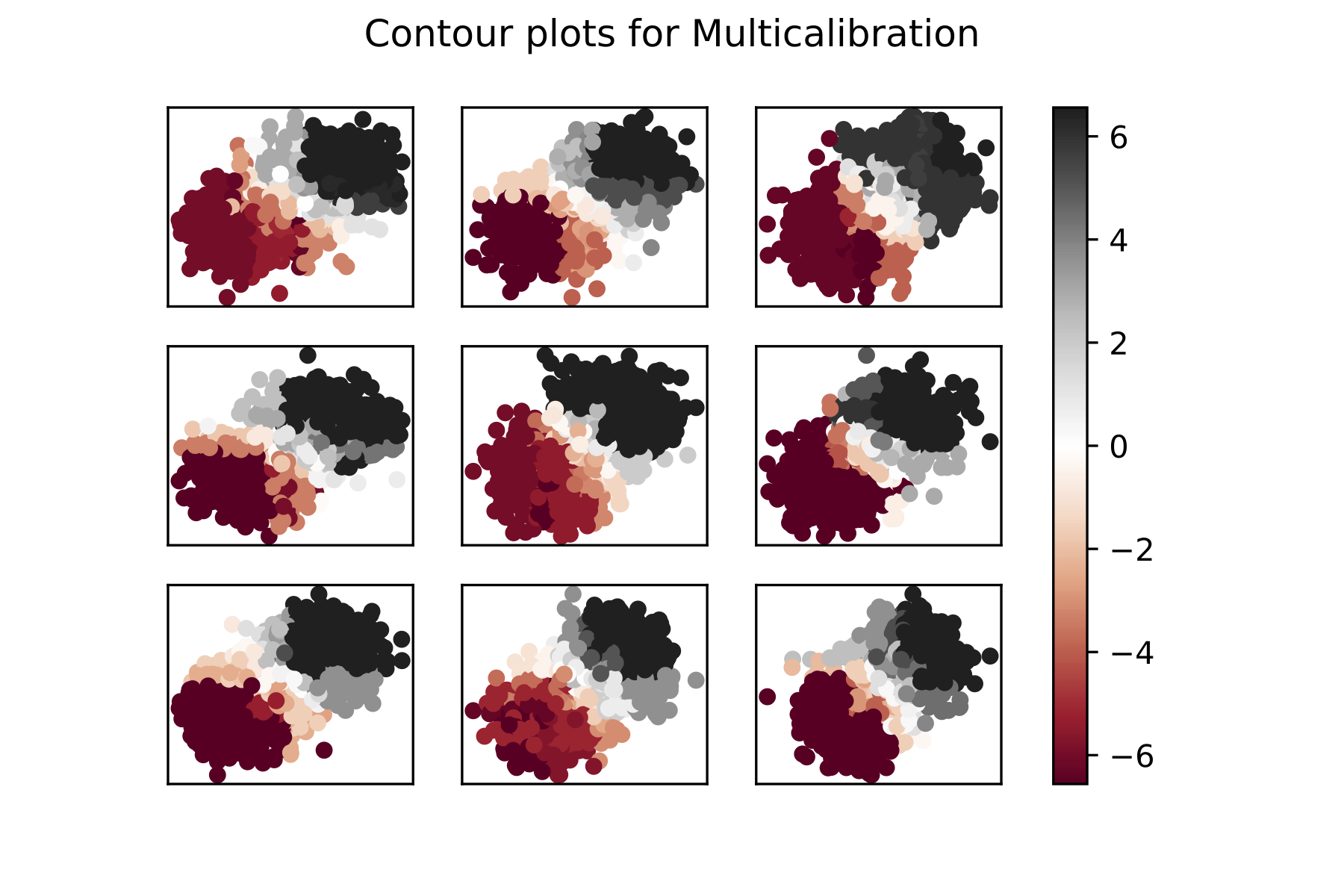}
    \caption{Color indicates the importance weight\\ assigned by the multi-calibration algorithm.}
    \label{fig:contour_mc}
\end{subfigure}
 \caption{The first 9 pairs of Gaussians in the mixture of $k=10$ Gaussians.}
\end{figure}

\paragraph{MNIST based data.}
Our second set of experiments are  based on MNIST images \cite{lecun-mnisthandwrittendigit-2010}. For a bias variable $\delta \in (0.5,1)$ we define the distribution $P_\delta$ to sample an even digit with probability $\delta$ and an odd digit with probability $1-\delta$. $R_\delta$ samples an odd digit with probability $\delta$ and an even digit with probability $1-\delta$. Therefore, assuming the distributions of odd and even digits have disjoint support, we have $\KL{R_\delta}{P_\delta} = d(\delta, 1 - \delta)$ (see Equation \eqref{eq:bkl}). 
As before, we first test the algorithms on estimating the divergence across the entire population. The results are summarized in Table~\ref{tbl:kl-mnist}. We see that \MC and \ME are quite similar when the base family of classifiers is a random forest. The quality of \ME drops significantly when the classifier is weaker: depth $1$ decision tree or logistic regression. We see that \KLIEP and \ulsif are significantly worse.

	\begin{table*}[h] 
	\caption{ KL estimation for MNIST}\label{tbl:kl-mnist}
		\small
		\centering
		\begin{tabular}{ c c c c c c c c c c }
			\hline
			bias & KL & LLK-RF5 & MC-RF5 & LLK-DT1 & MC-DT1 & LLK-LR & MC-LR & KLIEP & uLSIF\\
			\hline
			\hline
			0.6  & 0.12  & 0.07   & 0.07   & 0.04  & 0.05 & 0.04 & 0.04 & 0.04 & 0.05\\
			0.7  & 0.49  & 0.36   & 0.31    & 0.22  & 0.27 & 0.24 & 0.26 & 0.15 & 0.19\\
			0.8  & 1.2    & 0.92   & 0.89   & 0.55  & 0.79 & 0.64 & 0.72 & 0.43 & 0.42\\
			0.9  & 2.56  & 1.93   & 1.96   & 1.18  & 1.72 & 1.39 & 1.53 & 0.78 & 1.08\\
			0.95 & 3.85 & 2.96   & 2.85   & 1.55  & 2.58 & 1.9   & 2.22 & 1.32 & 1.16\\
			\hline
		\end{tabular}
	\end{table*}

In this data set we consider single digits as sub-populations. Note that the digits are not perfectly classified by our base classifiers, so strictly speaking they are not part of subsets for which the algorithm is multi-calibrated. Our goal here is precisely to test how well our predictions hold in settings which do not strictly conform to our assumptions. In our experiment we set the bias $\delta = 0.9$. We then set a sub-population $C$ to be images of two consecutive digits. Since one digit is odd and one digit is even we have $\KL{R_C}{P_C} = \KL{R}{P} = 2.55$. We tested all $10$ sub populations of this form. Results are depicted in Fig.~\ref{fig:mnist_subgroup}. We see \MC and \ME perform reasonably well, and much better than the rest of the algorithms. Interestingly, all algorithms struggle with the digits $4$ and $8$.

\begin{figure}
    \centering
    \includegraphics[width=0.65\textwidth]{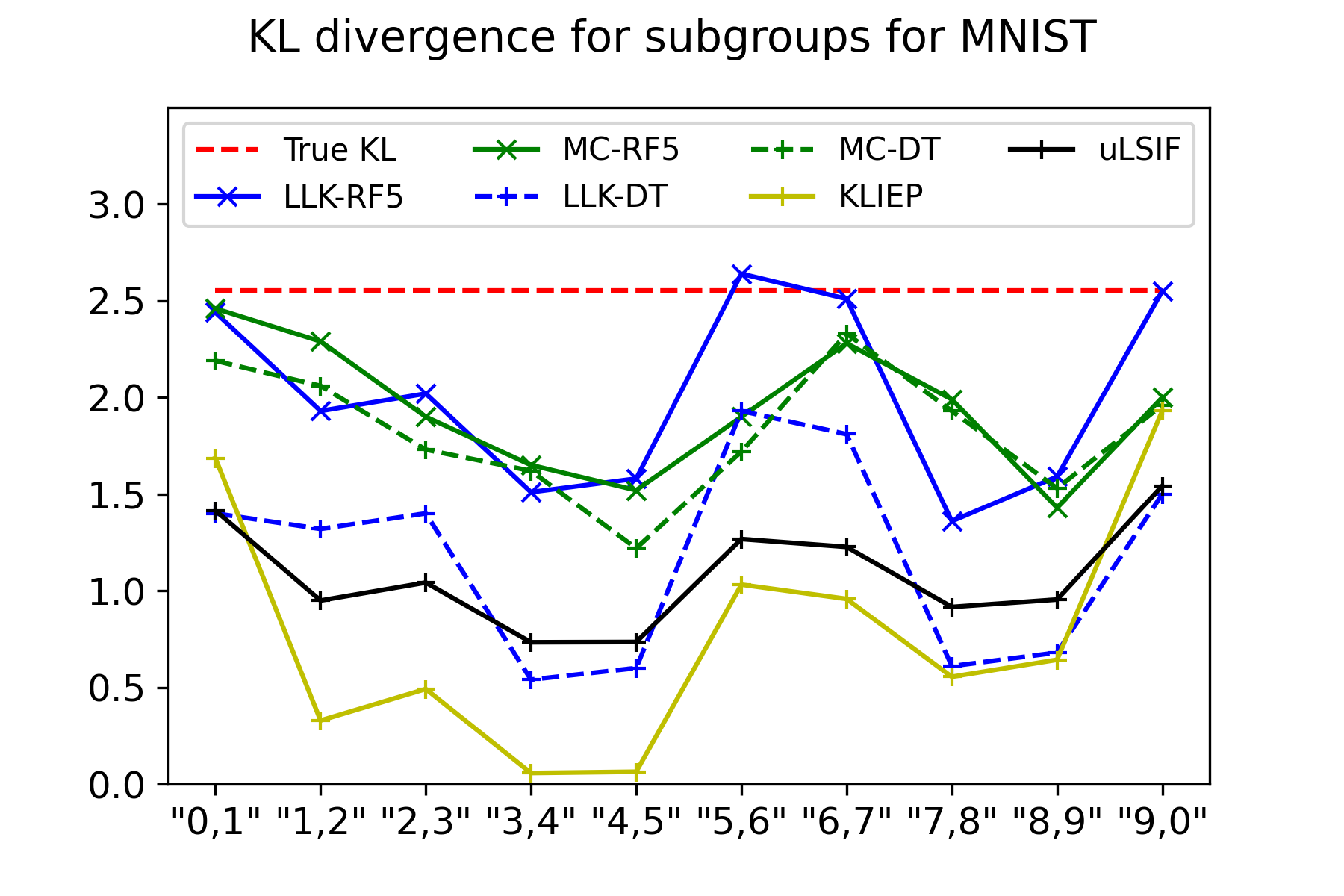}
   \caption{KL estimates for sub-populations, each sub-population is a pair of an odd and an even digit.}    \label{fig:mnist_subgroup}
\end{figure}

\eat{
Conclusion: When the feature family is relatively weak the multi-calibrated estimator is the best. In fact, it provides quite good estimates even when the feature set is decision trees of depth $1$. It requires many samples though, so when the classifier family becomes too strong it starts to over-fit, so when the depth is $10$ we actually see a drop in the estimated divergence as compared to depth $5$. The maxEnt (or LL-KLEIP) approach over fits later, so it provides the best results when depth is $10$. KLEIP is by far the weakest.  }

\section{Related work and discussion}
\label{sec:related}

Technically, our work builds on the notion of multi-calibrated partitions introduced in \cite{gopalan2021multicalibrated}, which in turn was inspired by the work of \cite{hkrr2018} on multi-calibration in supervised settings (also see \cite{kearns2018preventing,kgz}). \cite{gopalan2021multicalibrated} was motivated by completeness and soundness guarantees for sub-populations for the weights $w(x)$. They prove {\em sandwiching bounds} for $w(x)$ that are analogous to those shown in Corollary  \ref{cor:sandwich} for $\log(w(x))$, but bounds for $w(x)$ do not imply bounds for $\log(w(x))$ and vice versa.  While similar in motivation, technically the two works are quite different. Working with $\log(w(x))$ is cleaner, and allows us to connect the problem to KL estimation and Pythagorean theorems. There does not appear to be a similar connection to the Renyi divergence for $w(x)$.

There has been plenty of work on  density ratio estimation  \cite{sugiyama2012adensity, KLIEP, kanamori2010theoretical, sugiyamaBook} or equivalently importance weight estimation \cite{ cortes2008sample, cortes2010learning}. Kernel based approaches for estimating importance weights have also been proposed, starting with Kernel Mean Matching (KMM) introduced in \cite{huang2007correcting}, see also  \cite{cortes2008sample}, which generalize moment-matching methods \citep{qin1998inferences}. 
Point-wise accurate importance weight estimation is impossible in the worst case \cite{SBD}. This motivated the work of \cite{gopalan2021multicalibrated} which seeks to gives guarantees for a family of sets $\mC$ which represent sub-populations in the data.
There has been a lot of work on estimating KL and other divergences from random samples, see \cite{nguyen2007nonparametric, nguyen2010estimating, suzuki2008approximating, wang2005divergence, wang2009divergence}. The latter works also use a data-dependent partitioning scheme, but the details are rather different from ours. There is extensive work on estimating the KL-divergence under much stronger assumptions on the distributions \cite{hero2001alpha,hershey2007approximating,durrieu2012lower,NEURIPS2019_3147da8a}. 

Another line of relevant work is the use of boosting algorithms for distribution learning \cite{boostingBook}: algorithms such as \cite{KLIEP, dudik2007maximum} can be viewed as adaptations of the exponential-update rule of Adaboost \cite{freund1997decision, boostingBook} to density-ratio estimation. These algorithms ensure multi-accuracy for sub-populations in $\mC$. The algorithm of \cite{gopalan2021multicalibrated} can be viewed as adapting the Boosting by Branching Programs algorithm of \cite{mansour2002boosting, KearnsM99} to density-ratio estimation. In the supervised learning context, boosting by branching programs is known to have superior noise tolerance properties \cite{kalai2005boosting, KalaiMV08, LongS10}. 

Pythagorean theorems are studied in information theory under the subject of information geometry \cite{CTbook}. The term Pythagorean theorem is variously used to mean both equalities \cite{della1997inducing, boostingBook} or one-sided inequalities \cite{CTbook}.


\bibliography{references}
\bibliographystyle{alpha}

\newpage
\appendix
\section{Some Additional Results and Proofs}
\label{app:mcab}

\subsection{Proofs from Section \ref{sec:intro}}

\begin{proof}[Proof of Lemma \ref{lem:chain-rule}]
We have
\begin{align*}
    \KL{R}{P} &= \E_R\logf{R(x)}{P(x)} \\
    &= R(C)\E_{R|_C}\logf{R(x)}{P(x)} + R(\bar{C})\E_{R|_{\bar{C}}} \logf{R(x)}{P(x)}\\
    & = R(C)\E_{R|_C}\logf{R|_C(x)R(C)}{P|_C(x)P(C)} + R(\bar{C})\E_{R|_{\bar{C}}}\logf{R|_{\bar{C}}(x)R(\bar{C})}{P|_{\bar{C}}(x)P(\bar{C})} \\
    &= R(C)\E_{R|_C}\logf{R|_C(x)}{P|_C(x)} + 
R(\bar{C})\E_{R|_{\bar{C}}}\logf{R|_{\bar{C}}(x)}{P|_{\bar{C}}(x)}\\
& + R(C)\E_{R|_C}\logf{R(C)}{P(C)} +  R(\bar{C})\logf{R(\bar{C})}{P(\bar{C})}\\
&= R(C)\KL{R|_C}{P|_C} + R(\bar{C})\KL{R|_{\bar{C}}}{P|_{\bar{C}}} + \d(R(C), P(C)).
\end{align*}

\end{proof}

The following Lemma proves that distinguishing whether the KL divergence between two distributions $P$ and $R$ is 0, or almost as large as it could be (notice that if $R(x)/P(x)\le t\; \forall\; x \in \X$ then $\KL{R}{P} \le \log(t)$) needs a sample complexity polynomial in the size of the domain. The proof follows by a simple application of the birthday paradox, we note that it is possible to prove stronger lower bounds using better constructions \cite{BatuFRSW13,Valiant11} but we include the following result below for completeness.

\begin{lemma}\label{lem:impossibility}
\eat{
Let $P$ and $R$ be two distributions supported on some domain $\X$ of size $|\X|$. Say for any $t>0$ we are guaranteed that $R(x)/P(x)\le t\; \forall\; x \in \X$.  Then any algorithm for distinguishing whether $\KL{R}{P}=0$ or $\KL{R}{P} \ge \log(t)/10$ with success probability at least $2/3$ from i.i.d. samples from $P$ and $R$ requires at least $\Omega(\sqrt{|\X|})$ samples. }
Given any $t>0$ and a domain $\X$ of size $|\X|$, 
an  algorithm that, given two distributions $P,R$ over $X$ such that $R(x)/P(x) \le t$  can distinguish whether $\KL{R}{P}=0$ or $\KL{R}{P} \ge \log(t)/10$ with success probability at least $2/3$ from i.i.d. samples from $P$ and $R$, requires at least $\Omega(\sqrt{|\X|})$ samples.  
\end{lemma}
\begin{proof}
We first consider the case when $P$ and $R$ are both uniform distributions supported on half the domain (and the supports of both $P$ and $R$ are unknown to the algorithm). Note that by the birthday paradox, we not not expect to see any repetitions in $\sqrt{\X}/10$ samples drawn from a uniform distribution over a support of size $|\X|/2$, with probability $9/10$. Therefore, with $\sqrt{\X}/20$ samples drawn from $P$ and $R$, with probability $9/10$ we do not expect to any repetitions in the samples in either the case when $R=P$, or $R$ is a uniform distribution over a different support. Therefore no algorithm can distinguish between the case when $R=P$ or when the support of $R$ is drawn randomly and independently of $P$ with success probability more than $1/10$ given $O(\sqrt{\X})$ samples.

We can now leverage this lower bound to show a lower bound for any $t$ and some pair of distributions $R'$ and $P'$ such that $R'(x)/P'(x)\le t\; \forall\; x \in \X$. Let $U$ be the  uniform distribution over $\X$. For $P$ and $R$ as defined in the previous paragraph, let $P'=(2/t)U+(1-2/t)P$ and $R'=(2/t)U+(1-2/t)R$. Notice that in this case $R'(x)/P'(x)\in \{1,t-1\}\; \forall\; x \in \X \implies R'(x)/P'(x) \le t\; \forall\; x \in \X$. We note that if the support  of $R$ is chosen randomly and independently of the support of $P$, then by a Chernoff bound with probability $9/10$ the overlap in their supports is at most $|\X|/3$, which implies that $\KL{R'}{P'}\ge \log(t)/10$ with probability $9/10$. We now observe that if there exists an algorithm to distinguish whether $\KL{R'}{P'}=0$ or $\KL{R'}{P'} \ge \log(t)/10$ with success probability at least $p$ with $O(\sqrt{\X})$ samples, then it can be used to distinguish whether $R=P$ or $R$ is a uniform distribution over a different support as in the previous setup with success probability at least $(p-1/10)$ with $O(\sqrt{\X})$ samples. To verify this, observe that it is easy to generate $m$ samples from $P'$ and $R'$ given $m$ samples from $P$ and $R$. Therefore, by the lower bound in the previous paragraph, no algorithm can distinguish whether $\KL{R'}{P'}=0$ or $\KL{R'}{P'} \ge \log(t)/10$ with success probability at least $2/3$ with $O(\sqrt{\X})$ samples.

\end{proof}

\subsection{Claims and Proofs from Section \ref{sec:mcab}}

\begin{lemma}
\label{lem:exp-close}
    $Q$ is $\alpha$-\mact for $(R, \mC)$. 
\end{lemma}
\begin{proof}
Using items (1) and (2) of Lemma \ref{lem:partition}, we can write
\begin{align*}
    Q(C) = \sum_{i \in [m]} R(S_i)P|_{S_i}(C) = \E_{\s \sim R}[P|_{\s}(C)], \    R(C) = \E_{\s \sim R}[R|_{\s}(C)]
\end{align*} 
Hence 
\begin{align*}
    \abs{Q(C) - R(C)} &= \abs{\E_{\s \sim R}[P|_{\s}(C)- R|_\s(C)]} \leq \E_{\s \sim R}\lt[\abs{P|_{\s}(C) - R|_\s(C)}\rt] \leq \alpha.
\end{align*}
\end{proof}

\section{Additional Experiments} \label{sec:app-experiments}
We estimated $\KL{R}{P}$, when $R,P$ are mixtures of $k$ Gaussians. In Table~\ref{tbl:kl-mixture-2} we take $N=500,000$ samples of two dimensional Gaussians. \ulsif is not in the table since it creates a $N\times N$ kernel matrix and hence does not scale to this number of samples. In Table~\ref{tbl:kl-mixture-3} we take $5$ dimensional Gaussians. In all these tables we see \MC providing the most accurate estimate, especially when $k$ increases. Surprisingly, \MC estimates the divergence quite well even when the set of base classifiers is just threshold over basic featers (MC-DT1). 
\begin{table}
		\caption{ KL estimations for mixture of $k$ Gaussians with $N=500000$}\label{tbl:kl-mixture-2}		\centering
		\begin{tabular}{ c c c c c c c }
			\hline
			$k$ & KL & LLK-RF5 & MC-RF5 & LLK-DT1 & MC-DT1 & KLIEP \\
			\hline
			\hline
			5  & 9.02  & 7.66   & 7.79   & 2.43  & 7.54  & 4.95 \\
			10  & 9.02  & 6.09  & 7.49    & 1.23  & 7.34  & 4.88 \\
			15  & 9.02    & 3.69   & 7.23  & 0.40  & 6.61  & 5.07\\
			\hline
		\end{tabular}
	\end{table}
	
\begin{table}
		\caption{ KL estimations for mixture of $k$ Gaussians with $d=5$, $N=500000$}\label{tbl:kl-mixture-3}
		\centering
		\begin{tabular}{ c c c c  c c c }
			\hline
			$k$ & KL & LLK-RF5 & MC-RF5 & LLK-DT1 & MC-DT1 & KLIEP \\
			\hline
			\hline
			5  & 8.11  & 6.76   &  6.30  & 1.94  & 5.50  & 4.16 \\
			10  & 8.11  & 6.12  & 6.05    & 0.80  & 4.74 & 3.34 \\
			15  & 8.11    & 4.35   & 5.80  & 0.39  & 3.59 & 3.59\\
			\hline
		\end{tabular}
	\end{table}

	\begin{table}
		\caption{ Standard deviations across the $k$ subgroups for KL estimations for mixture of $k$ Gaussians with $N=30000$, averaged over 5 trials} \label{tbl:mixture-sd}
		\centering
		\begin{tabular}{ c  c c c c c c c c }
			\hline
			$k$ & LLK-RF5 & MC-RF5 & LLK-DT1 & MC-DT1 & KLIEP & ULSIF\\
			\hline
			\hline
			5  & $0.42 \pm 0.15$  & $0.41 \pm 0.13$   & $1.94 \pm 0.65 $  & $0.36 \pm 0.13$  & $0.35 \pm 0.24$ & $0.28 \pm 0.15$ \\
			10   & $2.27 \pm 0.24 $  & $0.57 \pm 0.21$    &  $1.34 \pm 0.24$ & $0.58 \pm 0.27$ & $0.49 \pm 0.16$ & $0.51 \pm 0.15$\\
			15     & $2.54 \pm 0.11$  & $0.60 \pm 0.19$   & $0.93 \pm 0.18$  & $1.16 \pm 0.23$ & $0.92 \pm 0.30$ & $0.74 \pm 0.25$\\
			\hline
		\end{tabular}
	\end{table}
	\begin{figure}
    \centering
    \includegraphics[width=0.5\textwidth]{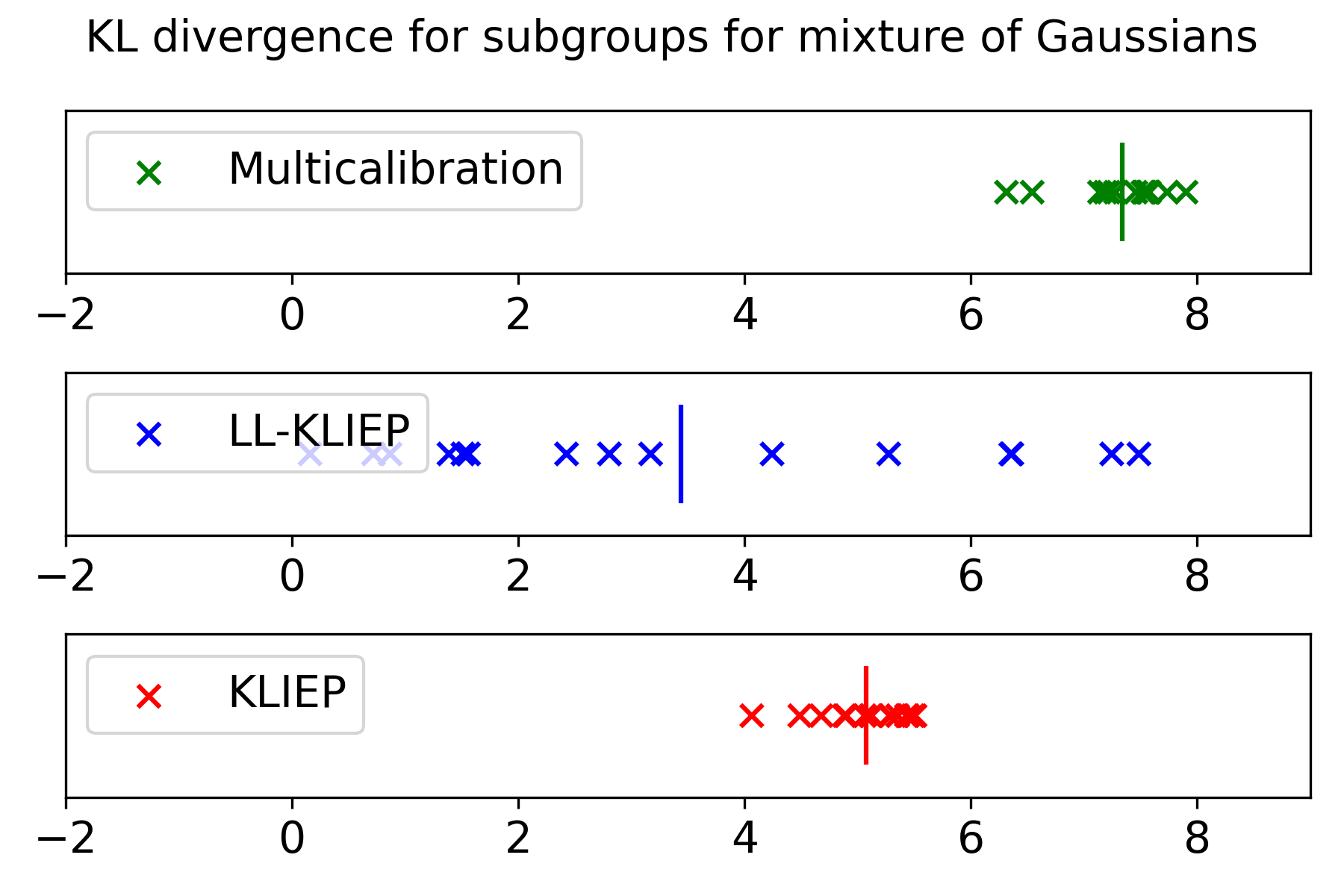}
  \caption{KL estimations for the $15$ subgroups for mixture of $15$ Gaussians with $N=500000$}    \label{fig:mixture-sub-2}
\end{figure}

\paragraph{Computation time and resources.} All our experiments are run on a standard laptop computer with 8 GB RAM. Our Python implementations of \MC\ and\ME\ have not been optimized at all for runtime efficiency. Still, we find the \MC\ algorithm to be reasonably efficient. To provide some ballpark numbers for the runtimes, MC-RF5 on the mixture of Gaussians experiment with $d=2,k=15, N=30000$ takes around 10 seconds to fit the data. The \ME\ implementation is slower, and takes around 100 seconds to fit this data. The MATLAB implementations of KLIEP and uLSIF take around 1 second to fit this data.

\paragraph{Choice of hyperparameters.} The hyperparameters in the \MC\ algorithm are the choice of classifier family, the width of the branching program, and the advantage below which we do not split a node. The choice of the classifier family is mentioned in each of the experiments. We set the width of the program to be 60, and the advantage to be $0.02$ for all the experiments. The hyperparameters for the LL-KLIEP algorithm are similar: the classifier family (as above, mentioned in the experiments), the advantage below which the program terminates (set to be $0.02$ as before), and the learning rate (set to be $0.02$). The hyperparameters for the KLIEP and uLSIF algorithms (width of the kernel and regularization) are optimized using the automatic hyperparameter tuning and cross-validation routines provided in the code.

\end{document}